\documentclass{article}

\usepackage{tabularx}
\usepackage{amssymb}
\usepackage{amsmath}
\usepackage{amsthm}

\usepackage{arxiv}

\newtheorem{assumption}{Assumption}
\newtheorem{proposition}{Proposition}

\usepackage[utf8]{inputenc} % allow utf-8 input
\usepackage[T1]{fontenc}    % use 8-bit T1 fonts
\usepackage{hyperref}       % hyperlinks
\usepackage{url}            % simple URL typesetting
\usepackage{booktabs}       % professional-quality tables
\usepackage{amsfonts}       % blackboard math symbols
\usepackage{nicefrac}       % compact symbols for 1/2, etc.
\usepackage{microtype}      % microtypography
\usepackage{graphicx}
\graphicspath{ {./images/} }

\title{(Human) Attention Is (Still) All You Need: Human oversight makes AI-assisted social science reliable}

\author{
 Chen Zhu \\
  China Agricultural University \\
  \texttt{zhuchen@cau.edu.cn} \\
   \And
 Xiaolu Wang \\
  China Agricultural University \\
  \texttt{wangxiaolu77@cau.edu.cn} \\
   \And
 Weilong Zhang \\
  University of Cambridge \\
  \texttt{weilong.zhang@econ.cam.ac.uk} \\
}

\begin{document}
\maketitle
\begin{abstract}
Large language models (LLMs) are increasingly delegated tasks once reserved for trained researchers: generating hypotheses, choosing specifications, drafting conclusions. Whether this delegation produces trustworthy science is not solely a technical question about model capability; it also depends on how cognitive labour is structured between humans and machines. We study this problem by organising an AI-assisted research workflow around behavioural-science principles, pre-commitment, decision sequencing, accountability, and attention allocation, rather than by directly measuring human psychology at the gates. We propose that reliability in AI-assisted research is a property of \emph{decision architecture}: the placement, sequencing, and binding force of the choices that humans and AI components are each permitted to make. In a pre-specified $2\times4$ factorial experiment ($N=280$ complete research runs across four datasets), an unconstrained multi-agent baseline produced critical failures in 72\% of runs; the same underlying model and the same agent decomposition, with identical prompts on the reasoning agents shared by both arms, failed in 16\% once organised by three architectural commitments (LLMs restricted to reasoning, data and estimation executed deterministically, three human decision gates; Fisher's exact $p<0.001$). Reliability gains were largest on the least publicly represented dataset, a Qing-dynasty population register, in a pattern consistent with a task-based production model with Fr\'{e}chet-distributed output quality. An 80-run ablation suggests that deterministic computation and human gates contribute independently, with exploratory evidence consistent with complementarity. We interpret HLER as a research harness rather than an autonomous AI scientist. Its contribution is to reduce failures sharply, make residual weaknesses more visible, and prevent unreliable claims from being advanced as publication-ready outputs.
\end{abstract}

%==============================================================
\section{Introduction}

Large language models are increasingly delegated tasks once reserved for trained researchers: generating hypotheses, choosing specifications, writing code, drafting conclusions\cite{lu2026ai_scientist,wu_autogen_2023,hong_metagpt_2023}, a shift from task-level assistance towards workflow-level research automation\cite{tie2026autoresearch}. Whether this delegation produces trustworthy science depends not on what models can do alone, but on how cognitive labour is divided between humans and machines. The behavioural failure modes that motivated a decade-long credibility movement in empirical social science (specification search, motivated interpretation, fragile causal claims) arise when probabilistic judgement is permitted to govern stages of research that demand deterministic discipline\cite{osc_2015_reproducibility,brodeur_mass_2020,silberzahn2018many,nosek_preregistration_2018}. LLMs, by construction, are probabilistic reasoners. Letting them operate end-to-end across a research pipeline does not eliminate these failure modes; it amplifies them.

The amplification is structural, not incidental. A human researcher's cognitive bandwidth places an upper bound on the size of the search space they can exploit; LLMs lift that bound by orders of magnitude. A human can p-hack a handful of specifications, whereas an LLM agent can do so across thousands. A human occasionally misremembers a citation; LLMs hallucinate references as a baseline behaviour\cite{ji_hallucination_2023}. Independent evidence from large-scale human--LLM comparisons converges on a consistent picture: LLM outputs cluster more tightly than human outputs, underperform humans at the upper tail of creative-task distributions, and homogenise content when used as ideation aids\cite{wang_uzzi_2025_creativity,doshi_hauser_2024,padmakumar_he_2024,anderson_homogenisation_2024}. The behavioural-science literature on choice architecture has long argued that defaults, sequencing, and commitment devices, rather than exhortation, discipline biased decision-makers\cite{thaler_sunstein_nudge}; the team-science literature similarly shows that human--machine combinations underperform when task complexity and division of labour are not strategically structured\cite{vaccaro_almaatouq_malone_2024,marks_etal_2005_multiteam}. The same logic applies to AI participants in empirical research: the question is not whether the model is well-aligned, but whether the workflow within which it operates is.

We treat reliability in AI-assisted research, then, as a property of \emph{decision architecture}\cite{shneiderman_2020,amershi_guidelines_2019,mosqueira2023hitl}: the placement, sequencing, and binding force of the choices that humans and AI components are each permitted to make. We organise our argument around three commitments. First, LLMs handle reasoning-intensive, exploratory tasks (hypothesis generation, identification critique, interpretation) where their probabilistic flexibility is an asset. Second, data construction and statistical estimation execute through deterministic code, where reproducibility is non-negotiable and human discretion is itself a liability. Third, humans hold three decisive gates (research question, identification strategy, publication) where contextual judgement, scope conditions, and ultimate accountability for claims cannot be delegated. A task-based production model with Fr\'{e}chet-distributed output quality (Methods) yields a sharp prediction from these commitments: the reliability dividend from human gates should be largest precisely on tasks furthest from the LLM's training distribution, where probabilistic flexibility is least constrained by familiarity. Our experiment does not measure the micro-behaviour of human gatekeepers, for example, time spent, attention patterns, or disagreement among multiple PIs. Instead, it tests whether a workflow organised around behavioural-science principles can reduce failure when the same LLM is placed inside different decision architectures.

We test the prediction with a pre-specified $2\times4$ factorial experiment. We implement the architectural commitments in HLER (Human-in-the-Loop Economic Research)\cite{zhu2026hler}, a modular multi-agent system, and compare it against an unconstrained baseline that uses the same underlying language model, the same agent decomposition, and identical prompts for the reasoning agents shared by both arms, but allows the LLM to operate across all stages without human gates. Four datasets (UK Biobank, China Health and Nutrition Survey, China Health and Retirement Longitudinal Study, and the historical China Multi-Generational Panel Dataset-Liaoning panel) deliberately vary in public scientific visibility and data structure, from widely used contemporary health datasets to Qing-dynasty population registers. Because LLM model's training corpus is unobserved, we use PubMed literature prevalence as an exploratory proxy for dataset familiarity rather than as a direct measure of $\theta_t$. Each of 280 runs produces either a complete research output or, in the constrained arm, a documented human-gate decision to withhold approval when a critical defect is identified. Outputs and gate decisions are independently evaluated by three experts using a pre-specified rubric covering feasibility, identification credibility, and output consistency.

The unconstrained pipeline produced critical failures in 72\% of runs. The same pipeline, with the architectural commitments imposed, failed in 16\% (Fisher's exact $p<0.001$). The reliability gain was largest on the least publicly represented dataset, the Qing-dynasty CMGPD panel, in a pattern consistent with the model; an 80-run ablation suggests that deterministic computation and human gates contribute independently, with exploratory evidence consistent with complementarity. The same model, the same agent decomposition, and identical prompts on the shared reasoning agents: only the architecture surrounding them moves, and a fourfold reduction in failure follows. We therefore conceptualise HLER as a research harness for AI-assisted empirical science: a decision architecture that channels LLM-generated reasoning through deterministic computation, explicit human gates, and auditable research records. The implication is behavioural rather than purely technical. Productive human--AI collaboration in empirical science is a problem of decision design, where architecture reduces, exposes, and contains failure rather than eliminating it.

%==============================================================
\section{Model}

\subsection{The HLER architecture}

HLER decomposes the empirical research workflow into eight specialised agent roles (data audit, data profiling, question generation, data construction, identification assessment, econometric estimation, manuscript drafting, and review) coordinated by an Orchestrator that maintains a persistent RunState recording active dataset, variable definitions, candidate questions, model specifications, and generated artefacts (Figure~\ref{fig:architecture}). Each agent is responsible for one block of the workflow and performs every sub-task within it. Two design principles distinguish the framework from existing autonomous research systems such as AI Scientist\cite{lu2026ai_scientist}, AutoGen\cite{wu_autogen_2023}, MetaGPT\cite{hong_metagpt_2023}, and CAMEL\cite{li_camel_2023}, all of which permit LLMs to operate across all pipeline stages without structured human oversight. First, agents are partitioned by operator type: \emph{deterministic} agents (data construction, estimation) execute reproducible code and emit the R scripts they used; \emph{probabilistic} agents (hypothesis generation, identification critique, interpretation) call the LLM. Second, three explicit human decision gates (research question selection, identification review, and publication decision) bind the workflow to commitments made before downstream results are visible.

\begin{figure}[t]
\centering
\includegraphics[width=0.95\linewidth]{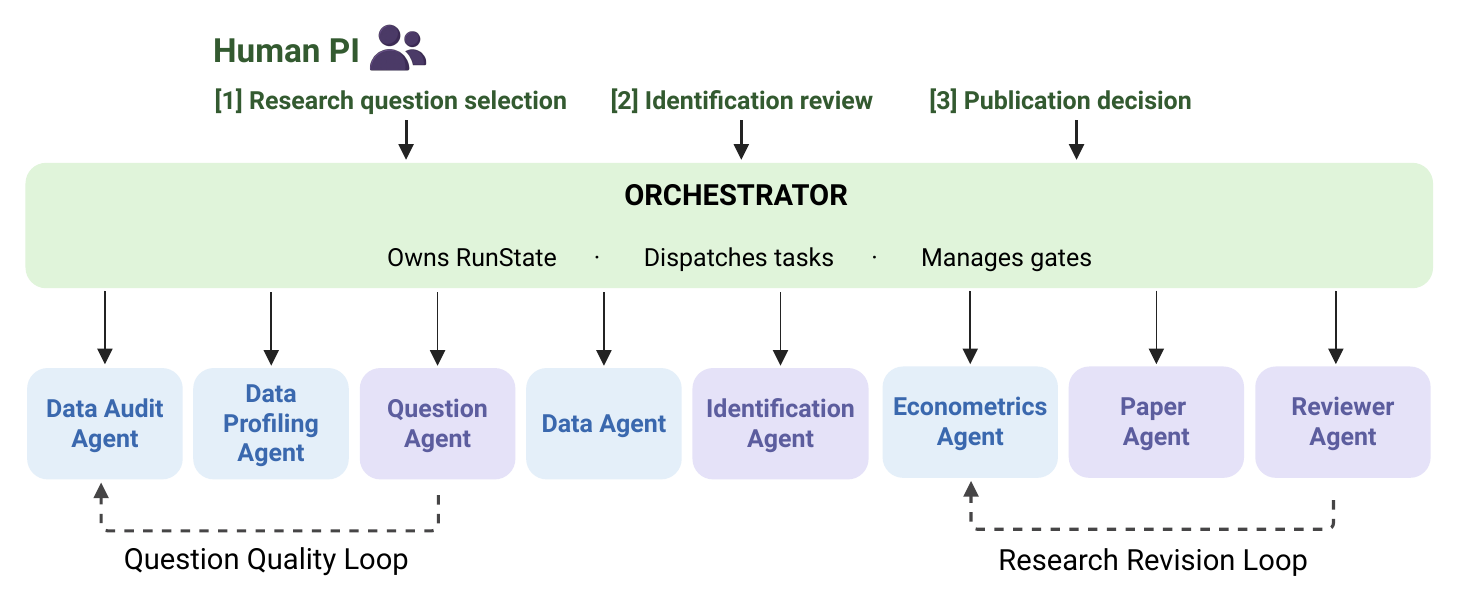}
\caption{\textbf{HLER decision architecture.} Specialised agents decompose the empirical research workflow. Purple-shaded agents are probabilistic (LLM-based reasoning); blue-shaded agents are deterministic (executable code). Human decision gates intervene at research-question selection, identification-strategy review, and publication decisions. The Orchestrator maintains the RunState, enforcing cross-stage consistency and producing an audit trail.}
\label{fig:architecture}
\end{figure}

\subsection{A task-based model of LLM research production}

The architecture above poses a single question: \emph{how should the researcher's limited time and attention be split between block-specific human gates and general oversight of the workflow as a whole?} We model this allocation problem and derive the optimal share in closed form. The model adapts the task-based production framework of Acemoglu and Restrepo\cite{acemoglu_restrepo_2018,acemoglu_restrepo_2022}, in which output is produced by allocating tasks across agents (humans, machines) of differing comparative advantage, and combines it with a Fr\'{e}chet productivity assignment of the kind used in trade and labour-allocation models\cite{eaton_kortum_2002,hsieh_hurst_jones_klenow_2019} to capture the stochastic, sample-then-select character of LLM output.

\paragraph*{Setup.}
Let $T=\{1,\ldots,m\}$ denote a finite set of research task-blocks. Each block $t$ is the responsibility of a single specialised LLM agent, which performs every sub-task in the block (data processing, identification, robustness, interpretation, and so on); ``block'' and ``agent'' are therefore in one-to-one correspondence. Working within block $t$, the agent generates $n_t$ candidate outputs and selects the best. We model the quality of each candidate as a Fr\'{e}chet draw,
\begin{equation}
q(o_{i,t})\;\overset{iid}{\sim}\;\mathrm{Fr\acute{e}chet}(\chi,\theta_t),
\label{eq:frechet}
\end{equation}
with common shape $\chi>0$ and task-specific scale $\theta_t>0$. The shape $\chi$ plays the role of \emph{effective temperature}: low $\chi$ produces a heavy-tailed quality distribution (exploratory, high-variance sampling, with a few candidates much better than the rest); high $\chi$ produces tightly clustered, low-variance outputs. The scale $\theta_t$ captures \emph{task--training-distribution proximity}: blocks whose tasks are well-represented in training have large $\theta_t$, while out-of-distribution blocks have small $\theta_t$. The candidate count $n_t$ formalises the sample-then-select strategy: by Fr\'{e}chet max-stability, the expected best-of-$n_t$ quality scales as $n_t^{1/\chi}$. The resulting gross block output $\widetilde{Q}_t$ thus depends on $q$, $n_t$, and $\chi$ but not on the oversight allocation; the symmetric-team assumption it relies on and the full derivation are deferred to Appendix~A.

\paragraph*{The allocation problem.}
We normalise the researcher's total oversight effort to one unit. For each block (equivalently, for the agent assigned to that block) the researcher chooses the share $\lambda_t\in[0,1]$ to spend on a block-specific \emph{human gate}; the remaining $1-\lambda_t$ is spent on \emph{general oversight} of the wider workflow. The two routes have distinct productivities, $\psi_A>0$ at the gate and $\psi_Z>0$ in general oversight, so that
\begin{equation}
a_t \;=\; \psi_A\,\lambda_t,
\qquad
\zeta_t \;=\; \psi_Z\,(1-\lambda_t).
\label{eq:budget}
\end{equation}
General oversight $\zeta_t$ raises overall quality multiplicatively, while gate investment $a_t$ reduces a coordination burden $\Gamma$ that grows with candidate count and specificity:
\begin{equation}
Q_t \;=\; \exp(\zeta_t)\,\widetilde{Q}_t\,\exp\!\big\{-\Gamma(n_t,\chi,a_t)\big\},
\qquad
\Gamma(n_t,\chi,a_t) \;=\; \chi\,n_t\,e^{-a_t}.
\label{eq:Qt}
\end{equation}
Eq.~\eqref{eq:Qt} adopts the simplest specification consistent with the qualitative claims: a unit scale on $\Gamma$, linear dependence on $n_t$, and unit elasticity of gate investment. Richer parameterisations (multiplicative scale, convex coordination cost, curvature on gate efficacy) leave the comparative statics below unchanged.

\paragraph*{Optimal allocation.}
Substituting Eq.~\eqref{eq:budget} into Eq.~\eqref{eq:Qt}, taking logs, and differentiating with respect to $\lambda_t$ gives the first-order condition
\begin{equation}
\frac{\partial \log Q_t}{\partial \lambda_t}
\;=\; -\psi_Z \;+\; \chi\,n_t\,\psi_A\,e^{-\psi_A \lambda_t} \;=\; 0,
\label{eq:foc}
\end{equation}
which solves to
\begin{equation}
\lambda_t^{\ast}
\;=\; \frac{1}{\psi_A}\,\log\!\left(\frac{\chi\,n_t\,\psi_A}{\psi_Z}\right).
\label{eq:lambda_star}
\end{equation}

Eq.~\eqref{eq:lambda_star} is the model's centrepiece. The optimal share of effort allocated to the block-specific human gate depends on exactly four parameters, each with a direct empirical interpretation:

\begin{enumerate}\itemsep2pt
\item[(i)] $\chi$, \emph{Fr\'{e}chet shape (effective temperature).} Higher $\chi$ tightens the candidate quality distribution and, mechanically through $\Gamma=\chi n_t e^{-a_t}$, scales the per-candidate coordination burden upward; both raise the gate share: $\partial\lambda_t^{\ast}/\partial\chi>0$. This $\chi$-scaled coordination burden cannot be managed through general oversight.

\item[(ii)] $n_t$, \emph{number of candidate outputs.} More candidates raise integration complexity, raising the gate share: $\partial\lambda_t^{\ast}/\partial n_t>0$.

\item[(iii)] $\psi_A$, \emph{gate productivity.} In the operating regime where gates are valuable ($\chi n_t \psi_A > e\,\psi_Z$), more productive gates substitute for raw effort and reduce the required share: $\partial\lambda_t^{\ast}/\partial\psi_A<0$.

\item[(iv)] $\psi_Z$, \emph{general-oversight productivity.} More productive general oversight pulls effort away from gates: $\partial\lambda_t^{\ast}/\partial\psi_Z<0$.
\end{enumerate}

\paragraph*{Full automation as a corner solution.}
Eq.~\eqref{eq:lambda_star} also tells us when block-specific gates should \emph{not} be deployed. The interior optimum $\lambda_t^{\ast}>0$ requires $\chi\,n_t\,\psi_A>\psi_Z$; whenever $\chi\,n_t\,\psi_A\le\psi_Z$, the first-order condition gives $\lambda_t^{\ast}\le 0$ and the constrained optimum is the corner $\lambda_t=0$, fully automated execution with no block-specific gate. This nests the operating regime of existing autonomous research systems\cite{lu2026ai_scientist,wu_autogen_2023,hong_metagpt_2023} as a special case: when blocks lie close to the LLM's training distribution (large $\theta_t$, low effective $\chi$), candidate counts are small, and gate productivity $\psi_A$ is low relative to general oversight $\psi_Z$, full automation is optimal. The framework therefore does not insist on human gates universally; it predicts the conditions under which they should be deployed and, equivalently, the conditions under which automation alone suffices.

A further empirical implication concerns the level rather than the slope of $Q_t$. Because $\theta_t$ enters $Q_t$ as a pure multiplicative pre-factor through $\widetilde{Q}_t$, it drops out of the FOC and the optimal share $\lambda_t^{\ast}$ in Eq.~\eqref{eq:lambda_star} is itself \emph{independent} of $\theta_t$. The reliability gain from imposing the architecture, however, is not. Treat a run as a failure when $Q_t$ falls below some threshold $\tau$; the failure rate at any allocation $\lambda_t$ is $\Pr(Q_t<\tau)$, which depends on $\theta_t$ non-linearly. When $\theta_t$ is small (out-of-distribution tasks), the unconstrained pipeline ($\lambda_t=0$) sits closer to or below $\tau$, so the gate-induced reduction in coordination burden $\Gamma$ lifts a larger mass of runs over the threshold. \emph{The model therefore predicts that the empirical gain from imposing decision architecture, measured as the failure-rate gap between configurations, will be largest where the data environment is most divergent from likely LLM training distributions}, the prediction we test next.

\section{Experimental design}

We test this prediction with a $2\times4$ factorial design crossing pipeline configuration (constrained HLER vs.\ unconstrained baseline) with four datasets that vary in scale, structure, and proximity to likely LLM training data. The datasets are: UK Biobank (UKB; ${\sim}500{,}000$ participants, biomedical and lifestyle data); China Health and Nutrition Survey (CHNS; ${\sim}30{,}000$ individuals, longitudinal household panel since 1989); China Health and Retirement Longitudinal Study (CHARLS; ${\sim}20{,}000$ adults aged 45+); and CMGPD-Liaoning, a historical demographic panel from Qing-dynasty population registers (${\sim}1749$--$1909$). The CMGPD panel deliberately stresses the system: its variable conventions and substantive concepts diverge sharply from anything LLMs are likely to have seen in training, the empirical analogue of low $\theta_t$ in the model.

The unconstrained baseline is not a strawman. It uses the same underlying language model (Claude Sonnet 4.6) and the same high-level agent decomposition for reasoning tasks shared with HLER. Prompts are held identical for the reasoning agents active in both arms. The arms differ by design in the execution and governance layer: data construction and estimation are delegated to LLM-generated code rather than to deterministic agents, and the gate-specific approval agents are omitted. In the unconstrained arm, gate prompts are not run silently and are not auto-approved; the orchestrator advances automatically from one stage to the next. This mirrors the operating principles of leading autonomous research systems\cite{lu2026ai_scientist,wu_autogen_2023,hong_metagpt_2023}, so the gap we document should be read as evidence about decision architecture, not as a verdict on the underlying model. Each run produces a complete research output. Three independent expert reviewers evaluate each output on feasibility (can the question be implemented given the data?), identification credibility (does the strategy meet standard criteria for causal inference?), and output consistency (are reported results internally consistent with the underlying estimation?). Inter-rater reliability is substantial (mean pairwise Cohen's $\kappa = 0.67$). Failures are classified into five primary modes (Table~\ref{tab:failure_modes}).

\begin{table}[t]
\centering
\small
\caption{Taxonomy of primary failure modes.}
\label{tab:failure_modes}
\begin{tabular}{p{4cm} p{10cm}}
\toprule
\textbf{Failure type} & \textbf{Description} \\
\midrule
Infeasible questions & Questions that cannot be implemented given the available data, including reliance on non-existent variables, insufficient within-sample variation, or incompatibility with the data structure (e.g., panel methods on cross-sectional data). \\
Data-processing and execution failures & Errors in data handling or computational execution: incorrect variable construction, sample-filtering mistakes, merge mismatches, or code that fails to run. \\
Identification failures & Empirical strategies that violate core causal-identification assumptions: implausible parallel trends, uncontrolled confounding, reverse causality, or weak/invalid instruments. \\
Hallucinated references and fabrications & Generation of plausible but unverifiable or incorrect factual content: non-existent or misattributed academic references, incorrect author--year combinations, or fabricated supporting evidence. \\
Interpretation inconsistencies & Mismatches between empirical results and their textual interpretation: discrepancies between regression outputs and reported coefficients, incorrect statements of statistical significance, or conclusions not supported by the underlying estimates. \\
\bottomrule
\end{tabular}
\end{table}

\subsection{Decision architecture sharply reduces failure}

Across 200 main-experiment runs, the constrained HLER pipeline outperformed the unconstrained baseline on every evaluation dimension (Table~\ref{tab:main_results}). Feasibility rose from 0.37 to 0.83, identification credibility from 0.31 to 0.65, and output consistency from 0.29 to 0.78 (all $p<0.001$, two-sided Fisher's exact test). The headline statistic is the overall failure rate: 72\% of unconstrained runs exhibited at least one critical failure, compared with 16\% under HLER.

It is worth noting that identification credibility remains the hardest dimension. Even under HLER, 35\% of constrained runs fail the identification criterion, compared with 69\% under the unconstrained baseline. We therefore interpret HLER as a failure-reduction and failure-containment architecture rather than as a system that guarantees reliable causal claims. The remaining identification failures typically arise when the available data cannot support the proposed causal contrast, when identifying assumptions such as parallel trends or exclusion restrictions are not credible, or when the research question is substantively interesting but underidentified by the available design.

\begin{table}[t]
\centering
\small
\caption{Main experiment: constrained vs.\ unconstrained pipeline (pooled across datasets).}
\label{tab:main_results}
\begin{tabular}{lccc}
\toprule
 & Constrained ($n=100$) & Unconstrained ($n=100$) & $p$-value \\
\midrule
Feasibility rate           & 0.83 & 0.37 & $<0.001$ \\
Identification credibility & 0.65 & 0.31 & $<0.001$ \\
Output consistency         & 0.78 & 0.29 & $<0.001$ \\
\midrule
Any failure                & 0.16 & 0.72 & $<0.001$ \\
\bottomrule
\end{tabular}
\par\medskip
\footnotesize
\textit{Notes:} Proportions of runs satisfying each criterion; $p$-values from two-sided Fisher's exact tests.
\end{table}

\subsection{Heterogeneity: gains are largest where training-distribution divergence is largest}

Table~\ref{tab:by_dataset} disaggregates failure rates by dataset. The constrained pipeline outperforms the unconstrained baseline in all four cases, but the gap varies across datasets. Because dataset-specific $\theta_t$ is not directly observed, we use literature prevalence as an exploratory proxy for dataset familiarity. Specifically, PubMed searches (as of June 8, 2026) for the dataset names returned 18,601 results for UK Biobank/UKB, 15,772 for the China Health and Nutrition Survey/CHNS, 5,175 for the China Health and Retirement Longitudinal Study/CHARLS, and 2 for China Multi-Generational Panel Dataset/CMGPD. This proxy is imperfect and likely overweights biomedical visibility, but it provides a transparent external measure of how frequently each dataset appears in public scientific discourse.

The largest reliability gap appears in CMGPD-Liaoning (0.16 vs.\ 0.88), the dataset with the lowest literature prevalence and the most historically specific variables and conventions. The smallest gap appears in CHARLS (0.20 vs.\ 0.64), whose standard panel structure is more amenable to automated analysis. UKB (0.12 vs.\ 0.64) and CHNS (0.16 vs.\ 0.72) lie between. This heterogeneity should therefore be interpreted as suggestive rather than as a direct test of $\theta_t$: the pattern is consistent with the model's prediction that the reliability dividend from decision architecture is larger when the data environment is less familiar to the LLM, but $\theta_t$ itself is proxied rather than observed.

\begin{table}[t]
\centering
\small
\caption{Any-failure rate by dataset and pipeline configuration.}
\label{tab:by_dataset}
\begin{tabular}{lccc}
\toprule
\textbf{Dataset} & \textbf{Constrained} ($n=25$) & \textbf{Unconstrained} ($n=25$) & $p$-value \\
\midrule
UK Biobank        & 0.12 & 0.64 & $<0.001$ \\
CHNS              & 0.16 & 0.72 & $<0.001$ \\
CHARLS            & 0.20 & 0.64 & $0.004$ \\
CMGPD-Liaoning    & 0.16 & 0.88 & $<0.001$ \\
\bottomrule
\end{tabular}
\par\medskip
\footnotesize
\textit{Notes:} Proportion of runs with at least one failure. $p$-values from two-sided Fisher's exact tests within each dataset.
\end{table}

The composition of failures was also informative (Table~\ref{tab:failure_breakdown}). Hallucinated references and fabrications were seven times more common in unconstrained runs than under HLER (21 vs.\ 3); interpretation inconsistencies, six times (18 vs.\ 3); infeasible questions and identification failures, between three- and four-fold. Data-processing failures were rare under both configurations (1 each), confirming that deterministic components are reliable when they are used.

\begin{table}[t]
\centering
\small
\caption{Primary failure modes among failed runs (main experiment, pooled).}
\label{tab:failure_breakdown}
\begin{tabular}{lcc}
\toprule
Failure type & Constrained & Unconstrained \\
\midrule
Infeasible questions                       & 4  & 17 \\
Data-processing and execution failures     & 1  & 1  \\
Identification failures                    & 5  & 15 \\
Hallucinated references and fabrications   & 3  & 21 \\
Interpretation inconsistencies             & 3  & 18 \\
\midrule
Total failures                             & 16 & 72 \\
\bottomrule
\end{tabular}
\par\medskip
\footnotesize
\textit{Notes:} Counts of failed runs assigned to a primary failure type.
\end{table}

\subsection{Deterministic computation and human gates contribute independently}

To isolate which architectural feature drives the reliability gain, we ran an 80-run ablation on CHNS and CHARLS crossing two binary factors, human gates (on/off) and deterministic data processing (on/off), with 10 trials per cell per dataset. CHNS and CHARLS were chosen because they are the most comparable pair, minimising dataset-driven heterogeneity.

The full HLER configuration achieved the lowest failure rate (0.20; Table~\ref{tab:ablation}). Removing human gates while retaining deterministic processing raised failures to 0.35; removing deterministic processing while retaining human gates raised them to 0.45. Removing both yielded 0.70, approaching the unconstrained baseline. Both features therefore appear to contribute independently. The increment from removing both jointly (0.50) exceeds the sum of removing each alone (0.15 + 0.25 = 0.40), a pattern consistent with complementarity between human gates and deterministic computation. However, because the ablation contains only 20 runs per condition, the interaction contrast is imprecise and should be interpreted as exploratory rather than confirmatory. The model rationalises why complementarity is plausible: gate investment $a_t$ lowers the coordination burden $\Gamma$, while the deterministic-probabilistic operator partition raises effective gate productivity $\psi_A$ on the blocks where the gate operates. Both inputs enter the optimal allocation jointly, so removing either may weaken the other's marginal contribution.

\begin{table}[t]
\centering
\small
\caption{Ablation study: failure rates by design feature (CHNS + CHARLS).}
\label{tab:ablation}
\begin{tabular}{lcc}
\toprule
 & Human gates & No human gates \\
\midrule
Deterministic data processing   & 0.20 ($n=20$) & 0.35 ($n=20$) \\
No deterministic data processing & 0.45 ($n=20$) & 0.70 ($n=20$) \\
\bottomrule
\end{tabular}
\par\medskip
\footnotesize
\textit{Notes:} Each cell reports the proportion of runs with any failure. Full HLER = deterministic data processing + human gates.
\end{table}

\subsection{Illustrative cases}

Aggregate statistics describe the size of the effect; mechanism is more visible in cases. We summarise two; full transcripts are in Appendix~C.

\paragraph*{Case 1: Escalating and rejecting a flawed identification strategy.}

In a CHNS run, the system proposed estimating the impact of China's New Rural Cooperative Medical Scheme on systemic inflammation using an event-study design comparing rural treated individuals to urban controls. The question was substantively plausible and the biomarker outcome was objective, so it passed the research-question gate.

At the identification stage, the deterministic event-study diagnostic revealed a clear violation of the parallel-trends assumption. Pre-treatment coefficients exhibited a large and statistically significant deviation ($t=-4$, $p<0.001$), and the joint pre-trend test rejected the identifying assumption. The important point is not that the LLM reviewers immediately recognised the violation with full force. They did not: early drafts understated the severity of the pre-trend failure and treated restricted-sample specifications as potential repairs. What changed across review iterations was the escalation of the issue from a reported diagnostic concern to a publication-blocking identification defect.

The violation could not be convincingly resolved. At the publication gate, the human PI declined to approve the study, concluding that the estimated effects could not be interpreted causally. This case therefore illustrates the function of HLER as a research harness: it does not require each LLM reviewer to be individually reliable on the first pass. Instead, deterministic diagnostics, repeated critique, auditable records, and a final human gate jointly prevent an unreliable causal claim from being advanced as publication-ready output. The model in Eq.~\eqref{eq:Qt} captures this through the gate investment $a_t=\psi_A\lambda_t$: a positive $\lambda_t$ raises $a_t$ and so reduces the coordination burden $\Gamma=\chi n_t e^{-a_t}$ that would otherwise push final output quality $Q_t$ below the publication-decision threshold $\tau$.

\paragraph*{Case 2: Methodological sophistication does not substitute for identification.}
A CHARLS analysis examined the effect of urban employee medical insurance (UEBMI) on outpatient healthcare utilisation among the elderly. The system implemented a double/debiased machine-learning (DDML) framework with individual fixed effects, using random forests to flexibly control for nonlinear confounding.

Despite the use of state-of-the-art methods, the identification strategy faced a fundamental challenge. Transitions into UEBMI coverage were closely tied to retirement, and retirement timing is itself endogenous to health shocks. As a result, the estimated positive association between insurance enrolment and outpatient spending could not be cleanly interpreted as causal: it remained observationally indistinguishable from adverse selection into coverage.

During the review process, the system identified these limitations and progressively revised the interpretation. Early drafts overstated the result as evidence of insurance-induced utilisation; later versions explicitly acknowledged the marginal statistical significance and the inability to rule out endogeneity. The final manuscript reframed the findings as within-individual associations rather than causal effects and substantially downgraded its policy claims.

This case highlights a critical limitation of automated and machine-learning-augmented approaches: methodological sophistication cannot compensate for weak identification. Even with flexible modelling and advanced estimation techniques, the absence of a credible source of exogenous variation constrains the interpretability of results. HLER addresses this not by guaranteeing valid identification, but by ensuring that such limitations are transparently identified and appropriately reflected in the final output.

%==============================================================
\section{Discussion}

Our findings document a structural mismatch between probabilistic language models and the epistemic requirements of empirical social science, and show that this mismatch can be reduced through decision architecture rather than through better models alone. The task-based production model rationalises this empirically: when LLM output quality is Fr\'{e}chet-distributed and the scale parameter $\theta_t$ falls with task--training-distribution divergence, human oversight becomes most valuable where the LLM is weakest. The empirical pattern is consistent with this implication, although not a direct measurement of $\theta_t$: the largest reliability dividend appears on the historical CMGPD panel, which has the lowest public literature prevalence and the most unfamiliar data conventions, whereas smaller gaps appear for more widely represented contemporary health datasets.

The central contribution is therefore not that HLER makes AI-assisted research fully reliable. It does not. Identification credibility remains the hardest dimension, and a substantial share of constrained runs still fail this criterion. The contribution is instead that HLER sharply reduces critical failures and changes where those failures occur in the workflow. Under the unconstrained baseline, weak designs, infeasible analyses, and inconsistent interpretations can be carried forward as publication-ready outputs. Under HLER, many of the same weaknesses are detected earlier, documented more explicitly, or stopped at a human gate. This distinction matters: reliability in AI-assisted science should be understood not only as the production of correct outputs, but also as the containment of incorrect ones.

\subsection{From model capability to research harnesses}

The dominant framing of LLMs in research positions reliability as a model-quality problem: better-aligned, better-grounded, less hallucinatory models will produce more trustworthy science. Our results suggest that this framing is incomplete. The same underlying model produces a 72\% failure rate under one allocation of cognitive labour and a 16\% failure rate under another. The variable that moves is not the model but the architecture surrounding it.

We therefore conceptualise HLER as a research harness rather than an autonomous AI scientist. A harness does not make failure impossible. Instead, it constrains where failure can occur, makes it observable, and prevents weak or inconsistent outputs from being advanced without human review. In this sense, the contribution of HLER is not absolute reliability but governed reliability: a structured environment in which LLM-generated reasoning is coupled to deterministic computation, explicit decision gates, and auditable research records. This reframes a substantial portion of the AI-in-science debate as a question of behaviourally informed system design, including decision sequencing, commitment devices, and the placement of human judgement, rather than purely an ML-engineering question. It also suggests a productive division of effort: model developers improve capabilities; methodologists, journals, and field-specific scholars design the research harnesses within which capabilities are responsibly used\cite{shneiderman_2020,amershi_guidelines_2019,mosqueira2023hitl,lazer_css_2020}.

The architectural account complements a recent distributional finding. Wang et al.\cite{wang_uzzi_2025_creativity} document that LLM outputs are tightly clustered while human outputs exhibit greater variance, with humans dominating the upper tail of divergent-creativity scores; parameter tuning (``temperature'') initially raises mean LLM performance but eventually collapses into incoherence, and creative-genius personas reduce performance further. The same Fr\'echet shape parameter $\chi$ that governs the tail of LLM output quality in our model rationalises this pattern: probabilistic flexibility lengthens the tail in both directions, so reliability gains require cutting off the bad tail rather than pushing the good one further. Where Wang et al.\ characterise the distributional gap, we identify a workflow-side lever: a research harness that samples broadly inside probabilistic blocks but binds downstream decisions through deterministic execution and human gates. The two findings are mutually reinforcing: parameter-level adjustments inside the model cannot substitute for structural intervention outside it.

\subsection{Human gates as architectural pre-registration}

The credibility movement has long advocated pre-registration, specification transparency, and the separation of exploratory and confirmatory analysis as voluntary commitments\cite{nosek_preregistration_2018,munafo_manifesto_2017,christensen_miguel_2018,simonsohn_spec_curve_2020}. Voluntary commitments are difficult to sustain. HLER's three human gates can be understood as the same commitments enforced architecturally: the PI cannot see estimation results before selecting a research question; cannot see final estimates before approving the identification strategy; and cannot publish without an explicit publication-decision step. In the model, this discipline operates through the allocation $\lambda_t$ in Eq.~\eqref{eq:lambda_star}: setting $\lambda_t=0$ leaves the coordination burden $\Gamma$ unmitigated and so admits low-quality runs, whereas a positive $\lambda_t$ binds attention to the gate before downstream estimates are visible. In practice, this makes bypassing the commitment require affirmative effort rather than virtue. The deeper point is that AI-assisted workflows offer a rare opportunity to embed methodological commitments in code, where they are enforced by default rather than relied upon as professional discipline.

\subsection{Behaviourally informed workflow design}

The architectural account should not be read as a direct behavioural measurement of human gatekeepers. We do not observe how long PIs spend at each gate, what information they attend to, or how different PIs would disagree over the same case. The behavioural contribution is instead at the level of workflow design. Three design features are drawn from behavioural science. First, gates are placed before downstream estimates are visible, mirroring the logic of pre-commitment devices: the workflow binds decisions to information available before the result is known. Second, the gates occur precisely where automation bias and over-reliance are likely to be most damaging: unfamiliar tasks where LLM outputs may appear fluent but lack reliable grounding. Third, the operator-type partition reduces process losses familiar from team-science research by assigning deterministic tasks to code and judgement-intensive tasks to human review. Thus, our claim is not that we identify the psychology of PI decisions, but that a workflow organised around behavioural-science principles can materially reduce failure in AI-assisted empirical research.

\subsection{Possible implications beyond HLER}

Although our experiment concerns LLM-assisted research, the deterministic-probabilistic distinction may also offer a useful conjecture for empirical research workflows more generally. Many human-led research failures have a similar structural form: judgement, interpretation, or motivated reasoning enters stages that require deterministic discipline, such as data construction, estimation, coefficient reporting, or version control. Conversely, some stages genuinely require probabilistic judgement, including question framing, interpretation, and assessment of scope conditions. We do not test this broader claim here. The experiment evaluates AI-assisted workflows, and the theoretical model is explicitly built around LLM-generated candidate outputs with Fr\'{e}chet-distributed quality. The extension to human-led research should therefore be read as a conceptual implication and a direction for future work, not as an empirical conclusion of the present study. Still, the architecture suggests a useful diagnostic question for research design: which workflow stages should be locked down by deterministic execution, and which should remain open to structured human judgement?

\subsection{Limitations}

Several limitations qualify these findings. First, the four datasets do not exhaust the diversity of empirical social science: experimental data, text-as-data, network data, and qualitative sources present challenges outside our evaluation. Second, results are obtained with a single underlying model (Claude Sonnet 4.6); other models will have different failure profiles, and the specific rates we report are likely model-dependent, though we expect the qualitative pattern to generalise. Third, the deterministic-probabilistic boundary is not always sharp in practice: stages such as variable selection contain both elements. Fourth, the ablation uses smaller cell sizes than the main experiment, with 20 runs per condition, and its results should be read as indicative of the relative importance of design features rather than as precise interaction estimates. In particular, the complementarity pattern is exploratory and not powered for a definitive interaction test. Fifth, our reliance on expert reviewers introduces unavoidable judgement, although this reflects how empirical research is evaluated in practice. Sixth, our metrics (feasibility, identification credibility, output consistency) are necessary but not sufficient conditions for replicable findings. Finally, the model is a stylised production environment: the Symmetric Execution Team assumption and the i.i.d.\ Fr\'{e}chet draw simplify a richer reality in which prompts, dependencies across blocks, and learning during a run all matter.

\subsection{Outlook}

The future of AI in empirical social science depends less on whether models can generate research-like outputs (they can) than on how human and machine cognition are jointly arranged. Properly harnessed, AI can extend researchers' cognitive reach and even improve human decisions by surfacing novel options the human would not have considered\cite{lazer_css_2020,shin_griffiths_2023_superhuman}; left unharnessed, it amplifies the failure modes the credibility movement has spent a decade containing\cite{simonsohn_spec_curve_2020,brodeur_mass_2020,osc_2015_reproducibility}. The architectural disciplines we identify, separating reasoning from computation, committing to questions and identification strategies before estimation, and gating decisions at epistemically decisive points, are continuous with the reforms the field has been pursuing for human-led research. Productive human--AI collaboration in science is, in the end, a problem of decision design: not how to remove humans from the loop, but how to place human attention where it can most effectively reduce, expose, and contain failure.

%==============================================================
\section{Methods}

\subsection{System implementation}

HLER is implemented as a modular multi-agent system in Python. An Orchestrator manages workflow execution and maintains a persistent RunState recording the active dataset, variable definitions, candidate questions, model specifications, and generated artefacts. Eight specialised agents (data audit, data profiling, question generation, data construction, identification assessment, econometric estimation, manuscript drafting, and review) handle distinct functions. Probabilistic agents call the LLM (Claude Sonnet 4.6 via API); deterministic agents execute R scripts. The system is model-agnostic at the agent interface. All LLM calls and code executions are logged.

\subsection{Datasets and access}

We use four publicly available datasets. (i) UK Biobank (UKB), a large-scale biomedical resource with ${\sim}500{,}000$ participants and linked health, genetic, and lifestyle data. (ii) China Health and Nutrition Survey (CHNS), a longitudinal household panel covering ${\sim}30{,}000$ individuals across repeated waves since 1989. (iii) China Health and Retirement Longitudinal Study (CHARLS), a nationally representative longitudinal survey of ${\sim}20{,}000$ Chinese adults aged 45+. (iv) CMGPD-Liaoning, a historical demographic panel from Qing-dynasty population registers (${\sim}1749$--$1909$). Datasets were chosen to span sample size, temporal scope, geographic context, and variable structure, providing a demanding test of generalisability.

\subsection{Experimental design and pre-specified experimental protocol}

The main experiment uses a $2\times4$ factorial design crossing pipeline configuration (constrained, unconstrained) with dataset, with 25 runs per cell (200 runs total). The ablation crosses two binary factors, human gates (on/off) and deterministic data processing (on/off), on CHNS and CHARLS, with 10 runs per cell per dataset (80 runs total). The full experiment comprises 280 runs. The design, run protocol, evaluation rubric, and primary hypotheses were pre-specified before data collection.

In the constrained configuration, LLMs are restricted to reasoning tasks; data processing uses deterministic R code; and human gates are active at question selection, identification review, and publication decision. In the unconstrained baseline, the same LLM backbone and the same prompts are used for reasoning agents active in both arms, but data construction and estimation are delegated to LLM-generated code, and gate-specific approval agents are omitted. The bypassed gate prompts are neither run silently nor automatically approved; rather, the orchestrator advances to the next stage without a gate-checking step. Each run produces a structured research record consisting of a research question, empirical strategy, estimation results, and draft manuscript when the pipeline advances to that stage. In the constrained HLER arm, the record may instead end with a documented human-gate decision not to approve publication when the PI identifies an unresolved feasibility, identification, or consistency problem.

\subsection{A note on the baseline}

The unconstrained configuration uses the same underlying language model and identical prompts on the reasoning agents shared with the constrained pipeline (hypothesis generation, identification critique, manuscript drafting, review). Gate-specific approval agents, however, are active only in HLER. They are removed in the unconstrained arm rather than run silently or auto-approved. The only differences are that (i) data construction and estimation use LLM-generated code rather than deterministic R agents, which necessarily changes the prompts on those two blocks, and (ii) human decision gates are absent, allowing autonomous progression from one stage to the next. This design mirrors the operating principles of leading autonomous research systems such as AI Scientist\cite{lu2026ai_scientist} and multi-agent frameworks such as AutoGen\cite{wu_autogen_2023} and MetaGPT\cite{hong_metagpt_2023}, which permit LLMs to operate across all pipeline stages without structured human oversight. The unconstrained baseline therefore approximates the performance achievable under current full-automation paradigms when applied to empirical social science. The substantial gap we document should be read as evidence about the value of architectural constraints, not about deficiencies in the underlying language model.

\subsection{Evaluation}

Outputs are independently assessed by three expert reviewers (PhD-level economists or quantitative social scientists, blinded to pipeline configuration) using a standardised rubric (Appendix~B). The same three reviewers evaluated all 280 runs, ensuring that differences across conditions were not driven by changes in reviewer composition. The rubric evaluates three dimensions: feasibility, identification credibility, and output consistency. Inter-rater agreement is summarised by Cohen's $\kappa$ (mean pairwise $\kappa = 0.67$ in the main experiment).

We distinguish between two concepts: output failure and failure containment. A run is coded as an output failure if the final approved or advanced research output contains at least one critical defect in feasibility, identification credibility, or output consistency. In the constrained HLER arm, a PI decision not to approve publication because a design flaw was detected is not coded as a pipeline failure; it is coded as a successful failure-containment event. This coding reflects the purpose of HLER: not to force every initiated run to produce a publishable claim, but to prevent unreliable claims from being advanced as publication-ready outputs.

Failures are classified into the five primary modes defined in Table~\ref{tab:failure_modes}. A single run may exhibit multiple modes; in Table~\ref{tab:failure_breakdown}, each failed run is assigned to its most consequential mode.

\subsection{Statistical analysis}

Primary comparisons use two-sided Fisher's exact tests on dichotomous failure indicators; proportions and exact $p$-values are reported. For pooled analyses across datasets we additionally fit logistic regressions with dataset fixed effects; estimated configuration effects are robust to this adjustment. The ablation is analysed by treating each of the four conditions as a level and reporting cell proportions. Because the ablation contains only 20 runs per condition, we treat interaction patterns as exploratory and do not interpret them as precise estimates of complementarity.

\bibliographystyle{unsrt}  
\bibliography{references}

\newpage
%==============================================================
\appendix
\section*{Appendix A: Symmetric execution and Fr\'{e}chet aggregation}

The main text states that the gross block output $\widetilde{Q}_t$ depends on $q$, $n_t$, and $\chi$ but not on the oversight allocation $\lambda_t$. This appendix gives the supporting assumption and derivation.

\begin{assumption}[Symmetric Execution Team]\label{ass:symmetric}
In equilibrium, the candidate outputs generated by an LLM agent for block $t$ are evenly spaced across the sub-tasks of the block and have common quality depth $q$, common Fr\'{e}chet shape $\chi$, and common generation cost $c$. Under this approximation, an agent generating $n_t$ candidates produces gross assignment value
\begin{equation}
\widetilde{Q}_t \;=\; q\,\chi^{\beta}\,n_t^{\alpha}\,c^{-\varphi},
\qquad
\alpha=\tfrac{1}{\chi},\;\;\beta>0,\;\;\varphi>0,
\label{eq:Qtilde_appendix}
\end{equation}
where $\alpha=1/\chi$ governs returns to candidate count and $\beta>0$ captures the curvature of $\widetilde{Q}_t$ in $\chi$. Eq.~\eqref{eq:Qt} in the main text uses the special case $c=1$ and absorbs $\chi^\beta$ into $q$, since $\widetilde{Q}_t$ does not enter the first-order condition for $\lambda_t$.
\end{assumption}

\begin{proposition}[Fr\'{e}chet aggregation to gross assignment value]\label{prop:aggregation}
Let each candidate output $o\in\mathcal{O}_t$ draw quality independently from $\mathrm{Fr\acute{e}chet}(\chi,\theta_t)$ as in Eq.~\eqref{eq:frechet}, and let generation cost $c_o=c$ be common across candidates. Then, under Assumption~\ref{ass:symmetric}, the expected maximum quality-adjusted output selected by the assignment rule $o_t(\tau)\in\arg\max_{o\in\mathcal{O}_t}\{z_o(\tau)\,c_o^{-\varphi}\}$, aggregated across sub-tasks $\tau\in t$, takes the reduced form
\begin{equation}
\mathbb{E}\!\left[\max_{o\in\mathcal{O}_t} z_o(\tau)\,c_o^{-\varphi}\right]
\;=\; \theta_t\cdot\chi^{\beta}\cdot n_t^{\alpha}\cdot c^{-\varphi}
\;\equiv\; \widetilde{Q}_t.
\label{eq:frechet_aggregation}
\end{equation}
\end{proposition}

\begin{proof}
\textit{Step 1: Max-stability of the Fr\'{e}chet distribution.} If $q(o_{i,t})\overset{iid}{\sim}\mathrm{Fr\acute{e}chet}(\chi,\theta_t)$, the maximum of $n_t$ independent draws satisfies
\begin{equation}
\max_{i=1,\ldots,n_t} q(o_{i,t}) \;\sim\; \mathrm{Fr\acute{e}chet}\!\left(\chi,\;n_t^{1/\chi}\,\theta_t\right),
\end{equation}
by the max-stability property of the Fr\'{e}chet family.

\textit{Step 2: Expected value of the maximum.} The mean of $\mathrm{Fr\acute{e}chet}(\chi,s)$ is $s\,\Gamma(1-1/\chi)$ for $\chi>1$. Hence
\begin{equation}
\mathbb{E}\!\left[\max_i q(o_{i,t})\right] \;=\; n_t^{1/\chi}\,\theta_t\,\Gamma\!\left(1-\tfrac{1}{\chi}\right).
\end{equation}

\textit{Step 3: Incorporating cost.} With common cost $c_o=c$,
\begin{equation}
\mathbb{E}\!\left[\max_o z_o(\tau)\,c_o^{-\varphi}\right] \;=\; c^{-\varphi}\,n_t^{1/\chi}\,\theta_t\,\Gamma\!\left(1-\tfrac{1}{\chi}\right).
\end{equation}

\textit{Step 4: Reduced-form parameterisation.} Setting $\alpha=1/\chi$, $q\equiv\theta_t\,\Gamma(1-1/\chi)$, and $\beta>0$ to capture the curvature of $\widetilde{Q}_t$ in the Fr\'{e}chet shape $\chi$,
\begin{equation}
\mathbb{E}\!\left[\max_o z_o(\tau)\,c_o^{-\varphi}\right] \;=\; q\,\chi^{\beta}\,n_t^{\alpha}\,c^{-\varphi} \;\equiv\; \widetilde{Q}_t. \qquad\square
\end{equation}
\end{proof}

\noindent\textit{Remarks.} (i) Quality depth $q=\theta_t\,\Gamma(1-1/\chi)$ combines the Fr\'{e}chet scale and shape, so $\theta_t$ enters the model only through $q$. (ii) Setting cost $c=1$ and absorbing $\chi^\beta$ into $q$ yields the simplified form $\widetilde{Q}_t = q\,n_t^{1/\chi}$ used in the main text; because $\widetilde{Q}_t$ is independent of $\lambda_t$, this does not affect Eq.~\eqref{eq:lambda_star}. (iii) Replacing the linear $\Gamma=\chi n_t e^{-a_t}$ in Eq.~\eqref{eq:Qt} with the more general $\Gamma=\kappa\,\chi\,n_t^{1+\eta}\,e^{-\phi a_t}$ for $\kappa,\eta,\phi>0$ leaves the comparative statics in Eq.~\eqref{eq:lambda_star} qualitatively unchanged: it scales the closed-form solution by constants but does not alter the signs of $\partial\lambda_t^{\ast}/\partial\chi$, $\partial\lambda_t^{\ast}/\partial n_t$, $\partial\lambda_t^{\ast}/\partial\psi_A$, or $\partial\lambda_t^{\ast}/\partial\psi_Z$.

%==============================================================
\section*{Appendix B: Standardised evaluation rubric}
\label{app:rubric}

Each generated output is evaluated independently by three expert reviewers along three dimensions: feasibility, identification credibility, and output consistency. Reviewers are instructed to apply the following criteria.

\subsection*{B.1 Feasibility}

Feasibility assesses whether the proposed research question and empirical strategy can be implemented given the available dataset.

\noindent An output is classified as \emph{feasible} if:
\begin{itemize}\itemsep0pt
\item All required variables are present in the dataset;
\item The proposed design is compatible with the data structure (e.g., panel methods applied to panel data);
\item The sample size and variation are sufficient to support the analysis.
\end{itemize}

\noindent An output is classified as \emph{infeasible} if any of the following apply:
\begin{itemize}\itemsep0pt
\item Required variables are missing or incorrectly defined;
\item The empirical design is incompatible with the data (e.g., difference-in-differences without a valid comparison group);
\item The analysis cannot be executed in practice.
\end{itemize}

\subsection*{B.2 Identification credibility}

Identification credibility evaluates whether the empirical strategy satisfies standard criteria for causal inference.

\noindent An output is classified as \emph{credible} if:
\begin{itemize}\itemsep0pt
\item The identification strategy is clearly specified and justified;
\item Key assumptions (e.g., parallel trends, exogeneity, exclusion restriction) are plausible in the given context;
\item Potential confounders and alternative explanations are appropriately addressed.
\end{itemize}

\noindent An output is classified as \emph{not credible} if:
\begin{itemize}\itemsep0pt
\item The identification strategy is missing, unclear, or inappropriate;
\item Core assumptions are violated or unsupported;
\item The design relies on correlations that cannot be interpreted causally.
\end{itemize}

\subsection*{B.3 Output consistency}

Output consistency assesses whether the reported results are internally coherent.

\noindent An output is classified as \emph{consistent} if:
\begin{itemize}\itemsep0pt
\item Reported coefficients match the underlying estimation results;
\item Statistical significance is correctly interpreted;
\item The narrative accurately reflects the empirical findings;
\item Cited studies, factual claims, and references are verifiable and appropriately represented.
\end{itemize}

\noindent An output is classified as \emph{inconsistent} if:
\begin{itemize}\itemsep0pt
\item There are discrepancies between tables, coefficients, and text;
\item The interpretation contradicts the reported estimates;
\item Conclusions extend beyond what is supported by the results;
\item References are fabricated, misattributed, or unverifiable;
\item Factual claims are unsupported, invented, or inconsistent with the underlying analysis.
\end{itemize}

\subsection*{B.4 Evaluation procedure}

As shown in Figure~\ref{fig:eval_framework}, each generated research output is independently evaluated by three expert reviewers along three dimensions: feasibility, identification credibility, and output consistency. Each reviewer assigns a binary judgement (pass/fail) for each dimension. Final classification on each dimension is determined by majority rule (at least 2 of 3 reviewers). Inter-rater reliability is assessed using mean pairwise Cohen's $\kappa$.

For outputs that fail on one or more dimensions, reviewers additionally assign a primary failure mode using the five-category taxonomy in Table~\ref{tab:failure_modes}. The primary failure mode is intended to diagnose the dominant source of failure rather than to provide an exhaustive list of all defects in a given output.

\begin{figure}[h!]
\centering
\includegraphics[width=0.85\linewidth]{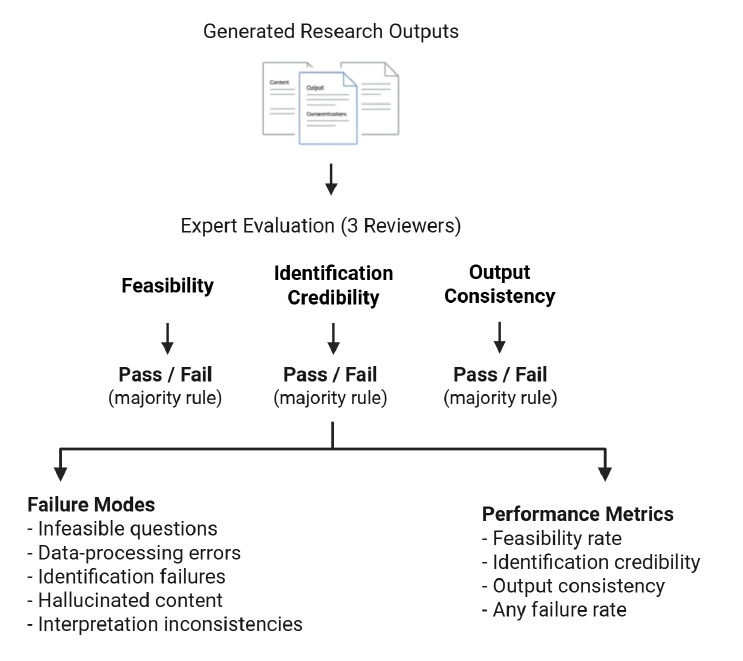}
\caption{\textbf{Evaluation framework for HLER outputs.} Each generated output is independently graded by three reviewers on feasibility, identification credibility, and output consistency. Each dimension is decided by majority rule. Failed outputs are assigned a primary failure mode from the five-category taxonomy; passed outputs contribute to the aggregate performance metrics.}
\label{fig:eval_framework}
\end{figure}

%==============================================================
\section*{Appendix C: Illustrative case studies}
\label{app:cases}

\subsection*{C.1 Case 1: Escalating and rejecting a flawed identification strategy}

We first examine a case evaluating the impact of China's New Rural Cooperative Medical Scheme (NRCMS) on systemic inflammation, measured by high-sensitivity C-reactive protein, using data from the China Health and Nutrition Survey (CHNS). The system proposed a difference-in-differences event-study design comparing rural residents (treated) to urban residents (controls) around the 2006 policy rollout. The question was retained at the selection stage due to its substantive relevance and the use of an objective biomarker outcome.

During the identification stage, the deterministic event-study diagnostic revealed a clear violation of the parallel-trends assumption. Pre-treatment coefficients exhibited a large and statistically significant deviation (e.g., $t=-4$, $p<0.001$), and the joint pre-trend test rejected the identifying assumption. The case is therefore not best read as evidence that LLM reviewers immediately detected an obvious problem. Rather, it shows how the harness prevents an obvious diagnostic from being normalised or rationalised away. Early drafts understated the severity of the pre-trend failure and treated restricted-sample specifications as possible repairs; later review iterations escalated the issue and made the identification failure explicit.

The violation could not be convincingly resolved. At the publication decision stage, the human PI declined to approve the study, concluding that the estimated effects could not be interpreted causally. Accordingly, this run is treated as a successful containment case rather than as a failed HLER output. This case illustrates a central function of HLER as a research harness: it does not assume that LLM-based reviewers are fully reliable on the first pass, but uses deterministic diagnostics, repeated critique, auditable records, and a final human gate to prevent the propagation of spurious causal claims.

\subsection*{C.2 Case 2: Methodological sophistication does not substitute for identification}

We next consider a case employing a more advanced estimation strategy, examining the effect of urban employee medical insurance (UEBMI) on outpatient healthcare utilisation among the elderly using CHARLS data. The system implemented a double/debiased machine learning (DDML) framework with individual fixed effects, using random forests to flexibly control for nonlinear confounding.

Despite the use of state-of-the-art methods, the identification strategy faced a fundamental challenge. Transitions into UEBMI coverage were closely tied to retirement, and retirement timing is itself endogenous to health shocks. As a result, the estimated positive association between insurance enrolment and outpatient spending could not be cleanly interpreted as causal, as it remained observationally indistinguishable from adverse selection into coverage.

During the review process, the system identified these limitations and progressively revised the interpretation. Early drafts overstated the result as evidence of insurance-induced utilisation, whereas later versions explicitly acknowledged the marginal statistical significance and the inability to rule out endogeneity. The final manuscript reframed the findings as within-individual associations rather than causal effects and substantially downgraded its policy claims.

This case highlights a critical limitation of automated and machine-learning--augmented approaches: methodological sophistication cannot compensate for weak identification. Even with flexible modelling and advanced estimation techniques, the absence of a credible source of exogenous variation constrains the interpretability of results. HLER addresses this not by guaranteeing valid identification, but by ensuring that such limitations are transparently identified and appropriately reflected in the final output.

\begin{table}[h]
\centering
\small
\caption{Primary failure modes nested within the three evaluation dimensions.}
\label{tab:failure_modes_appendix}
\begin{tabular}{p{4.2cm} p{10cm}}
\toprule
\textbf{Failure type} & \textbf{Description} \\
\midrule
Infeasible questions &
Research questions that cannot be implemented given the available data, including reliance on non-existent variables, insufficient within-sample variation, or incompatibility with the data structure (e.g., attempting panel methods on cross-sectional data). \\[2pt]
Data-processing and execution failures &
Errors in data handling or computational execution, such as incorrect variable construction, sample-filtering mistakes, merge mismatches, or code that fails to run as intended. These failures are rare but represent breakdowns in deterministic pipeline components. \\[2pt]
Identification failures &
Empirical strategies that violate core causal-identification assumptions, including implausible parallel trends, uncontrolled confounding, reverse causality, or weak and invalid instruments. \\[2pt]
Hallucinated references and fabrications &
Generation of plausible but unverifiable or incorrect factual content, including non-existent or misattributed academic references, incorrect author--year combinations, or fabricated supporting evidence. \\[2pt]
Interpretation inconsistencies &
Mismatches between empirical results and their textual interpretation, including discrepancies between regression outputs and reported coefficients, incorrect statements of statistical significance, or conclusions not supported by the underlying estimates. \\
\bottomrule
\end{tabular}
\end{table}
\clearpage
\newpage
%==============================================================
\section*{Data and Code Availability}

The GitHub repository at \url{https://github.com/maxwell2732/hler-working-papers} contains a curated selection of HLER-generated working papers together with the R scripts used to construct samples and produce estimates. Evaluation rubrics, example run logs, and the pre-specified document are included in Supplementary Information. All datasets used are publicly available: UK Biobank (\url{https://www.ukbiobank.ac.uk}), CHNS (\url{https://www.cpc.unc.edu/projects/china}), CHARLS (\url{https://charls.pku.edu.cn}), and CMGPD-Liaoning (available through the Lee-Campbell group).

\subsection*{Computational cost and reproducibility}

Each run requires approximately 20--25 minutes and \$0.8--\$1.5 in API costs. The HLER deterministic agents emit the executable R scripts used for sample construction and estimation, ensuring that every HLER-generated study is independently rerunnable. Code, scripts, prompts, evaluation rubrics, and example run logs are released publicly (see Data and Code Availability).

\subsection*{Use of AI tools in manuscript preparation}

Claude (Anthropic, Claude Opus 4.6) assisted with manuscript condensing, editing, and formatting. All scientific content, design choices, analyses, and conclusions are the authors' own.

\end{document}